\def\year{2022}\relax
\documentclass[letterpaper]{article} % DO NOT CHANGE THIS
\usepackage{aaai22}  % DO NOT CHANGE THIS
\usepackage{times}  % DO NOT CHANGE THIS
\usepackage{helvet}  % DO NOT CHANGE THIS
\usepackage{courier}  % DO NOT CHANGE THIS
\usepackage[hyphens]{url}  % DO NOT CHANGE THIS
\usepackage{graphicx} % DO NOT CHANGE THIS
\usepackage{natbib}  % DO NOT CHANGE THIS AND DO NOT ADD ANY OPTIONS TO IT
\usepackage{caption} % DO NOT CHANGE THIS AND DO NOT ADD ANY OPTIONS TO IT
\DeclareCaptionStyle{ruled}{labelfont=normalfont,labelsep=colon,strut=off} % DO NOT CHANGE THIS
\usepackage{color}
\usepackage{algorithm}
\usepackage{algorithmic}
\usepackage{mathtools}
\usepackage{amsthm}
\usepackage{multirow}
\usepackage{subfigure}
\usepackage{float}
\usepackage{wrapfig}
\usepackage{eucal}
\usepackage{adjustbox}
\usepackage{amsfonts}
\usepackage{booktabs}
\usepackage{multicol}

\newcommand{\algfont}[1]{\texttt{#1}}

\DeclareMathOperator*{\argmin}{arg\,min}
\newtheorem{theorem}{Theorem}
\newtheorem{corollary}{Corollary}
\newtheorem{lemma}{Lemma}
\newtheorem{assumption}{Assumption}

\usepackage{newfloat}
\usepackage{listings}
\floatstyle{ruled}
\newfloat{listing}{tb}{lst}{}
\floatname{listing}{Listing}

\title{FedProto: Federated Prototype Learning across Heterogeneous Clients}

\author{
    Yue Tan\textsuperscript{\rm 1},
    Guodong Long\textsuperscript{\rm 1}, 
    Lu Liu\textsuperscript{\rm 1}, 
    Tianyi Zhou\textsuperscript{\rm 2,3}, 
    Qinghua Lu\textsuperscript{\rm 4}, 
    Jing Jiang\textsuperscript{\rm 1}, 
    Chengqi Zhang \textsuperscript{\rm 1}
}
\affiliations {
    \textsuperscript{\rm 1} Australian Artificial Intelligence Institute, FEIT, University of Technology Sydney\\
    \textsuperscript{\rm 2} University of Washington,
    \textsuperscript{\rm 3} University of Maryland,
    \textsuperscript{\rm 4} Data61, CSIRO\\
    yue.tan@student.uts.edu.au,  $\{$guodong.long, jing.jiang, chengqi.zhang$\}$@uts.edu.au,\\ lu.liu.cs@icloud.com,
    tianyizh@uw.edu, qinghua.lu@data61.csiro.au
}

\usepackage{bibentry}

\begin{document}

\maketitle

\begin{abstract}
Heterogeneity across clients in federated learning (FL) usually hinders the optimization convergence and generalization performance when the aggregation of clients' knowledge occurs in the gradient space.
For example, clients may differ in terms of data distribution, network latency, input/output space, and/or model architecture, which can easily lead to the misalignment of their local gradients.
To improve the tolerance to heterogeneity, we propose a novel federated prototype learning (\algfont{FedProto}) framework in which the clients and server communicate the abstract class prototypes instead of the gradients. 
\algfont{FedProto} aggregates the local prototypes collected from different clients, and then sends the global prototypes back to all clients to regularize the training of local models.
The training on each client aims to minimize the classification error on the local data while keeping the resulting local prototypes sufficiently close to the corresponding global ones.
Moreover, we provide a theoretical analysis to the convergence rate of \algfont{FedProto} under non-convex objectives.
In experiments, we propose a benchmark setting tailored for heterogeneous FL, with \algfont{FedProto} outperforming several recent FL approaches on multiple datasets.
\end{abstract}

\section{Introduction}
Federated learning (FL) is widely used in multiple applications to enable collaborative learning across a variety of clients without sharing private data. It aims at training a global model on a centralized server while all data are distributed over many local clients and cannot be freely transmitted for privacy or communication concerns \cite{mcmahan2016communication}. The iterative process of FL has two steps: (1) each local client is synchronized by the global model and then trained using its local data; and (2) the server updates the global model by aggregating all the local models. Considering that the model aggregation occurs in the gradient space, traditional FL still has some practical challenges caused by the heterogeneity of \textit{data} and \textit{model} \cite{kairouz2019advances}. Efficient algorithms suitable to overcome both these two challenges have not yet been fully developed or systematically examined.

To tackle the \textit{statistical heterogeneity} of data distributions, one straightforward solution is to maintain multiple global models for different local distributions, e.g., the works for clustered FL \cite{sattler2020clustered}. Another widely studied strategy is personalized FL \cite{tan2021towards} where a personalized model is generated for each client by leveraging both global and local information. Nevertheless, most of these methods depend on gradient-based aggregation, resulting in high communication costs and heavy reliance on homogeneous local models. 

However, in real-world applications, \textit{model heterogeneity} is common because of varying hardware and computation capabilities across clients \cite{long2020federated}. Knowledge Distillation (KD)-based FL \cite{lin2020ensemble} addresses this challenge by transferring the teacher model's knowledge to student models with different model architectures. However, these methods require an extra public dataset to align the student and teacher models' outputs, increasing the computation costs. Moreover, the performance of KD-based FL can significantly degrade with the increase in the distribution divergence between the public dataset and on-client datasets that are usually non-IID.

Inspired by prototype learning, merging the prototypes over heterogeneous datasets can effectively integrate the feature representations from diverse data distributions \cite{snell2017prototypical,liu2020universal,dvornik2020selecting}. On-client intelligent agents in the FL system can share knowledge by exchanging information in terms of representations, despite statistical and model heterogeneity \cite{cai2020target, li2021deep}. For example, when we talk about ``dog", different people will have a unique ``imagined picture" or ``prototype" to represent the concept ``dog". Their prototypes may be slightly diverse due to different life experience and visual memory. 
Exchanging these concept-specific prototypes across people enables them to acquire more knowledge about the concept ``dog". Treating each FL client as a human-like intelligent agent, the core idea of our method is to exchange prototypes rather than share model parameters or raw data, which can naturally match the knowledge acquisition behavior of humans.

In this paper, we propose a novel prototype aggregation-based FL framework where only prototypes are transmitted between the server and clients. The proposed solution does not require model parameters or gradients to be aggregated, so it has a huge potential to be a robust framework for various heterogeneous FL scenarios. Concretely, each client can have different model architectures and input/output spaces, but they can still exchange information by sharing prototypes. Each abstract prototype represents a class by the mean representations transformed from the observed samples belonging to the same class. Aggregating the prototypes allows for efficient communication across heterogeneous clients.

Our main contributions can be summarized as follows:
\begin{itemize}
	\item We propose a benchmark setting tailored for heterogeneous FL that considers a more general heterogeneous scenario across local clients. 
	\item We present a novel FL method that significantly improves the communication efficiency in the heterogeneous setting. To the best of our knowledge, we are the first to propose prototype aggregation-based FL.
    \item We theoretically provide a convergence guarantee for our method and carefully derive the convergence rate under non-convex conditions.
	\item Extensive experiments show the superiority of our proposed method in terms of communication efficiency and test performance in several benchmark datasets.
\end{itemize}

\section{Related Work} \label{sec:related-work}

\subsection{Heterogeneous Federated Learning}
Statistical heterogeneity across clients (also known as the non-IID problem) is the most important challenge of FL. \algfont{FedProx} \cite{li2018federated} proposed a local regularization term to optimize each client's local model. Some recent studies \cite{arivazhagan2019federated,liang2020think,deng2020adaptive} train personalized models to leverage both globally shared information and the personalized part \cite{tan2021towards, jiang2020decentralized}. 
The third solution is to provide multiple global models by clustering the local models \cite{mansour2020three,ghosh2020efficient,sattler2020clustered} into multiple groups or clusters. Recently, self-supervised learning strategies are incorporated into the local training phase to handle the heterogeneity challenges \cite{li2021model, liu2021graph, yang2021interpretable}. \cite{fallah2020personalized} applies meta-training strategy for personalized FL.

Heterogeneous model architecture is another major challenging scenario of FL. The recently proposed KD-based FL \cite{lin2020ensemble,jeong2018communication,li2019fedmd,long2021federated} can serve as an alternative solution to address this challenge. In particular, with the assumption of adding a shared toy dataset in the federated setting, these KD-based FL methods can distill knowledge from a teacher model to student models with different model architectures. 
Some recent studies have also attempted to combine the neural architecture search with federated learning \cite{zhu2020federated,he2020fednas,singh2020differentially}, which can be applied to discover a customized model architecture for each group of clients with different hardware capabilities and configurations. A collective learning platform is proposed to handle heterogeneous architectures without access to the local training data and architectures in \cite{hoang2019collective}. Moreover, functionality-based neural matching across local models \cite{wang2020federated} can aggregate neurons with similar functionality regardless of the variance of the model architectures.

However, most of these mentioned FL methods focus on only one heterogeneous challenging scenario. All of them use gradient-based aggregation methods which will raise concerns about communication efficiency and gradient-based attacks \cite{zhu2019deep, chen2020robust, liu2021anomaly, zheng2021generative}.

\subsection{Prototype Learning}
The concept of prototypes (the mean of multiple features) has been explored in a variety of tasks. In image classification, a prototype can be a proxy of a class and is calculated as the mean of the feature vectors within every class \cite{snell2017prototypical}. In action recognition, the features of a video in different timestamps can be averaged to serve as the representation of the video \cite{twostream, xue2021dynamic}. Aggregated local features can serve as descriptors for image retrieval \cite{babenko2015aggregating}.
Averaging word embeddings as the representation of a sentence can achieve competitive performance on multiple NLP benchmarks \cite{wieting2015towards}.
The authors in \cite{hoang2020learning} use prototypes to represent task-agnostic information in distributed machine learning and propose a new fusion paradigm to integrate those prototypes to generate a new model for a new task. In \cite{michieli2021prototype}, prototype margins are used to optimize visual feature representations for FL.
In our paper, we borrow the concept of prototypes to represent one class and apply prototype aggregation in the setting of heterogeneous FL.

In general, prototypes are widely used in learning scenarios with a limited number of training samples \cite{snell2017prototypical}. This learning scenario is consistent with the latent assumption of cross-client FL: that each client has a limited number of instances to independently train a model with the desired performance. The assumption has been widely supported by the FL-based benchmark datasets \cite{caldas2018leaf,he2020fedml} and in related applications, such as healthcare \cite{rieke2020future,xu2020federated} and street image object detection \cite{luo2019real}.

\begin{figure*}[htbp!]
	\centering
	\setlength{\abovecaptionskip}{0.cm}
	\setlength{\belowcaptionskip}{-0.4cm}
	\includegraphics[width=0.9\textwidth]{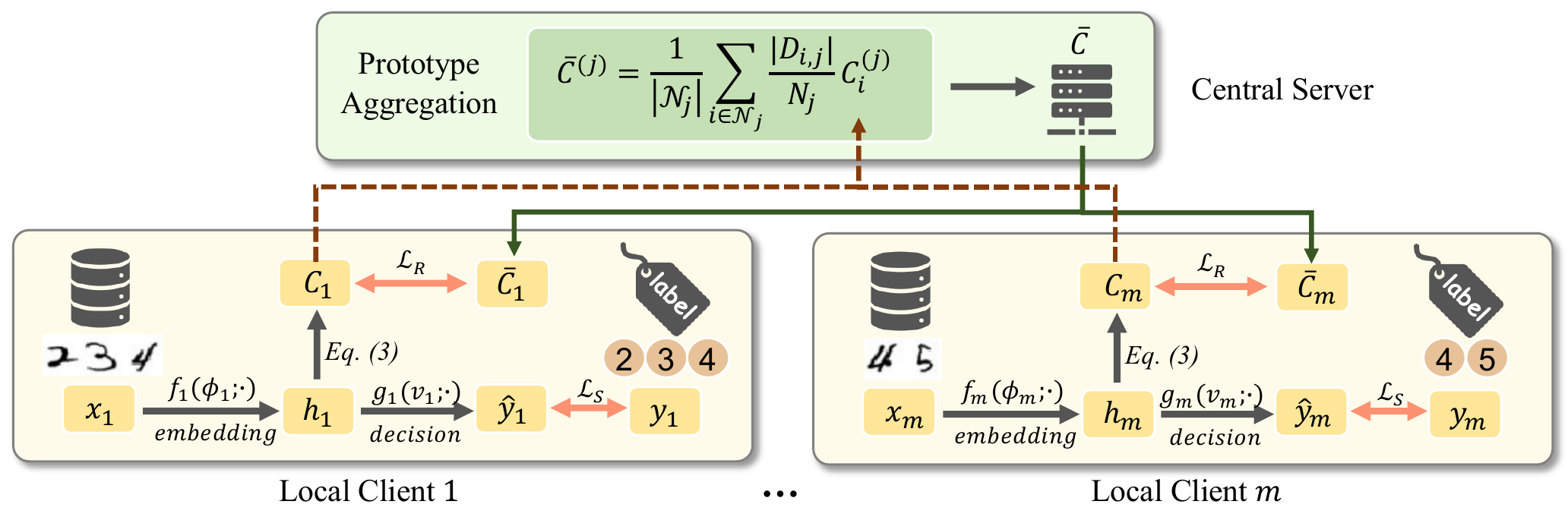}
	\caption{An overview of \algfont{FedProto} in the heterogeneous setting. For example, the first client is to recognize the digits ${2,3,4}$, while the $m$-th client is to recognize the digits ${4,5}$. First, the clients update their local prototype sets by minimizing the loss of classification error $\mathcal{L}_S$ and the distance between global prototypes and local prototypes $\mathcal{L}_R$. Then, the clients send their prototypes to the central server. The central server generates global prototypes and returns them to all clients to regularize the training of local models.}
	\label{framework}
\end{figure*}

\section{Problem Setting} \label{sec:problem}
\subsection{Heterogeneous Federated Learning Setting}
In federated learning, each client owns a local private dataset ${D}_i$ drawn from distribution $\mathbb{P}_i(x,y)$, where $x$ and $y$ denote the input features and corresponding class labels, respectively. Usually, clients share a model $\mathcal{F}(\omega;x)$ with the same architecture and hyperparameters. This model is parameterized by learnable weights $\omega$ and input features $x$. The objective function of \algfont{FedAvg}~\cite{mcmahan2016communication} is:
\begin{equation} \label{eq:sl-lossw}
\argmin_{\omega} \sum_{i=1}^{m} \frac{|D_i|}{N} \mathcal{L}_S(\mathcal{F}(\omega; x), y),
\end{equation}
\noindent where $\omega$ is the global model's parameters, $m$ denotes the number of clients, $N$ is the total number of instances over all clients, $\mathcal{F}$ is the shared model, and $\mathcal{L}_S$ is a general definition of any supervised learning task (e.g., a cross-entropy loss).

However, in a real-world FL environment, each client may represent a mobile phone with a specific user behavior pattern or a sensor deployed in a particular location, leading to statistical and/or model heterogeneous environment. In the statistical heterogeneity setting, $\mathbb{P}_i$ varies across clients, indicating heterogeneous input/output space for $x$ and $y$. For example, $\mathbb{P}_i$ on different clients can be the data distributions over different subsets of classes. In the model heterogeneity setting, $\mathcal{F}_i$ varies across clients, indicating different model architectures and hyperparameters. For the $i$-th client, the training procedure is to minimize the loss as defined below:
\begin{equation} \label{eq:fed-sp-loss}
\small
\argmin_{\omega_1,\omega_2,\dots,\omega_m} \sum_{i=1}^{m} \frac{|D_i|}{N} \mathcal{L}_S(\mathcal{F}_i(\omega_i; x), y).
\end{equation}

Most existing methods cannot well handle the heterogeneous settings above. In particular, the fact that $\mathcal{F}_i$ has a different model architecture would cause $\omega_i$ to have a different format and size. Thus, the global model's parameter $\omega$ cannot be optimized by averaging $\omega_i$. To tackle this challenge, we propose to communicate and aggregate prototypes in FL.%\looseness-1 
\subsection{Prototype-Based Aggregation Setting}
Heterogeneous FL focuses on the robustness to tackle heterogeneous input/output spaces, distributions and model architectures. For example, the datasets $D_{i}$ and $D_{k}$ on two clients $i$ and $k$ may take different statistical distributions of labels. This is common for a photo classification APP installed on mobile clients, where the server needs to recognize many classes $\mathbb{C}=\{{C}^{(1)}, {C}^{(2)}, \dots\}$, while each client only needs to recognize a few classes that constitute a subset of $\mathbb{C}$. The class set can vary across clients, though there are overlaps. 

In general, the deep learning-based models comprise two parts: (1) representation layers (a.k.a. embedding functions) to transform the input from the original feature space to the embedding space; and (2) decision layers to make a classification decision for a given learning task. 

\paragraph{Representation layers}
The embedding function of the $i$-th client is $f_i(\phi_i)$ parameterized by $\phi_i$. We denote $h_i=f_i(\phi_i;x)$ as the embeddings of $x$.

\paragraph{Decision layers} Given a supervised learning task, a prediction for $x$ can be generated by the function $g_i(\nu_i)$ parameterized by $\nu_i$. So, the labelling function can be written as $\mathcal{F}_i(\phi_i, \nu_i) = g_i(\nu_i) \circ f_i(\phi_i)$, and we use $\omega_i$ to represent $(\phi_i, \nu_i)$ for short.

\paragraph{Prototype}
We define a prototype ${C}^{(j)}$ to represent the $j$-th class in $\mathbb{C}$. For the $i$-th client, the prototype is the mean value of the embedding vectors of instances in class $j$,
\begin{equation}
\small
C_{i}^{(j)} = \frac{1}{|D_{i,j}|} \sum_{(x,y) \in D_{i,j}} f_i(\phi_i;x),
\label{eq:local_prototype}
\end{equation}

\noindent where $D_{i,j}$, a subset of the local dataset $D_i$, is comprised of training instances belonging to the $j$-th class.

\paragraph{Prototype-based model inference}
In the inference stage of the learning task, we can simply predict the label $\hat{y}$ to an instance $x$ by measuring the L2 distance between the instance's representational vector $f(\phi;x)$ and the prototype ${C}^{(j)}$ as follows:
\begin{equation}
\small
\hat{y} = \argmin_j || f(\phi;x) - C^{(j)} ||_2.
\end{equation}

\section{Methodology} \label{sec:methodology}

We propose a solution for heterogeneous FL that uses prototypes as the key component for exchanging information across the server and the clients. 

An overview of the proposed framework is shown in Figure \ref{framework}. The central server receives local prototype sets $C_1, C_2,  \ldots, C_m$ from $m$ local clients, and then aggregates the prototypes by averaging them. In the heterogeneous FL setting, these prototype sets overlap but are not the same. Taking the MNIST dataset as an example, the first client is to recognize the digits ${2,3,4}$, while another client is to recognize the digits ${4,5}$. These are two different handwritten digits set; nonetheless, there is an overlap. The server automatically aggregates prototypes from the overlapping class space across the clients. 

Using prototypes in FL, we do not need to exchange gradients or model parameters, which means that the proposed solution can tackle heterogeneous model architectures. Moreover, the prototype-based FL does not require each client to provide the same classes, meaning the heterogeneous class spaces are well supported. Thus, heterogeneity challenges in FL can be addressed. 
\vspace{-0.1cm}
\subsection{Optimization Objective}
\vspace{-0.05cm}
The objective of \algfont{FedProto} is to solve a joint optimization problem on a distributed network. \algfont{FedProto} applies prototype-based communication, which allows a local model to align its prototypes with other local models while minimizing the sum of loss for all clients' local learning tasks. The objective of federated prototype learning across heterogeneous clients can be formulated as
\begin{equation} \label{eq:lambda}
\small
\begin{aligned}
\argmin_{\left\{\bar{C}^{(j)}\right\}_{j=1}^{|\mathbb{C}|}}  \sum_{i=1}^m  \frac{|D_i|}{N} \mathcal{L}_S(\mathcal{F}_i(\omega_i;x), y) + \\
\lambda \cdot \sum_{j=1}^{|\mathbb{C}|} \sum_{i=1}^{m} \frac{|D_{i,j}|}{N_j} \mathcal{L}_R(\bar{C}_i^{(j)}, C_i^{(j)}),
\end{aligned}
\end{equation}
\vspace{-0.2cm}

\noindent where $\mathcal{L}_S$ is the loss of supervised learning (as defined in Eq.~\eqref{eq:fed-sp-loss}) and $\mathcal{L}_R$ is a regularization term that measures the distance (we use L2 distance) between a local prototype $C^{(j)}$ and the corresponding global prototypes $\bar{C}^{(j)}$. 
$N$ is the total number of instances over all clients, and $N_j$ is the number of instances belonging to class $j$ over all clients. 

The optimization problem can be addressed by alternate minimization that iterates the following two steps: (1) minimization w.r.t. each $\omega_i$ with $\bar{C}_i^{(j)}$ fixed; and (2) minimization w.r.t. $\bar{C}_i^{(j)}$ with all $\omega_i$ fixed. In a distributed setting, step (1) reduces to conventional supervised learning on each client using its local data, while step (2) aggregates local prototypes from local clients on the server end. Further details concerning these two steps can be seen in Algorithm~\ref{alg1}.

\begin{algorithm}[htbp!]
	\caption{\algfont{FedProto}} 
	\hspace*{0.02in} {\bf Input:} 
	$D_i, \omega_i, i=1,\cdots, m$\\
	\hspace*{0.02in} {\bf Server executes:}
	\begin{algorithmic}[1]
		\STATE Initialize global prototype set $\left\{\bar{C}^{(j)}\right\}$ for all classes.
		\FOR{each round $T = 1,2,...$} 
		\FOR{each client $i$ {\bf in parallel}}
		\STATE $C_i \leftarrow$ LocalUpdate $\left(i, \bar{C}_{i}\right)$
		\ENDFOR
		\STATE Update global prototype by Eq.~\ref{eq:proto_aggregation}.
		\STATE Update local prototype set $C_i$ with prototypes in $\{\bar{C}^{(j)}\}$
		\ENDFOR
	\end{algorithmic} 
	\hspace*{0.02in} \\
	\hspace*{0.02in} {\bf LocalUpdate}$\left(i, \bar{C}_i\right)$:
	\begin{algorithmic}[1]
		\FOR{each local epoch}
		\FOR{batch $\left(x_{i}, y_{i}\right) \in D_i$}
		\STATE Compute local prototype by Eq.~\ref{eq:local_prototype}.
		\STATE Compute loss by Eq.~\ref{eq:loss} using local prototypes.
		\STATE Update local model according to the loss.
		\ENDFOR
		\ENDFOR
		\RETURN $C^{(i)}$
	\end{algorithmic}
	\label{alg1}
\end{algorithm}

\vspace{-0.1cm}
\subsection{Global Prototype Aggregation}
\vspace{-0.05cm}
Given the data and model heterogeneity in the participating clients, the optimal model parameters for each client are not the same. This means that gradient-based communication cannot sufficiently provide useful information to each client. However, the same label space allows the participating clients to share the same embedding space and information can be efficiently exchanged across heterogeneous clients by aggregating prototypes according to the classes they belong to. 

Given a class $j$, the server receives prototypes from a set of clients that have class $j$. A global prototype $\bar{C}^{(j)}$ for class $j$ is generated after the prototype aggregating operation,
\begin{equation}
\small
\setlength\abovedisplayskip{0cm}
\setlength\belowdisplayskip{-0.1cm}
\bar{C}^{(j)}=\frac{1}{\left|\mathcal{N}_j\right|} \sum_{i\in{\mathcal{N}_j}} \frac{|D_{i,j}|}{N_j}C^{(j)}_{i},
\label{eq:proto_aggregation}
\end{equation}

\noindent where $C_i^{(j)}$ denotes the prototype of class $j$ from client $i$, and $\mathcal{N}_j$ denotes the set of clients that have class $j$.

\vspace{-0.1cm}
\subsection{Local Model Update}
\vspace{-0.05cm}
The client needs to update the local model to generate a consistent prototype across the clients. To this end, a regularization term is added to the local loss function, enabling the local prototypes $C_i^{(j)}$ to approach global prototypes $\bar{C}_i^{(j)}$ while minimizing the loss of the classification error. In particular, the loss function is defined as follows: 
\begin{equation}
\small
\mathcal{L}(D_i, \omega_i) = \mathcal{L}_S(\mathcal{F}_i(\omega_i;x), y) + \lambda \cdot \mathcal{L}_R\left(\bar{C}_i, C_i\right),
\label{eq:loss}
\end{equation}
\noindent where $\lambda$ is an importance weight, and $\mathcal{L}_R$ is the regularization term that can be defined as:
\begin{equation}
\small
\setlength\belowdisplayskip{-0.02cm}
\mathcal{L}_R = \sum_j d(C_i^{(j)}, \bar{C}_i^{(j)}),
\end{equation}
where $d$ is a distance metric of local generated prototypes $C^{(j)}$ and global aggregated prototypes $\bar{C}^{(j)}$. The distance measurement can take a variety of forms, such as L1 distance, L2 distance, and earth mover's distance. 

\vspace{-0.1cm}
\subsection{Convergence Analysis}
\vspace{-0.05cm}
We provide insights into the convergence analysis for \algfont{FedProto}. We denote the local objective function defined in Eq. \ref{eq:loss} as $\mathcal{L}$ with a subscript indicating the number of iterations and make the following assumptions similar to existing general frameworks \cite{wang2020tackling,li2018federated}. 
\begin{assumption}\label{as1}

{\rm(Lipschitz Smooth).} Each local objective function is $L_1$-Lipschitz smooth, which means that the gradient of local objective function is $L_1$-Lipschitz continuous,
\begin{equation}
\small
\begin{aligned}
\|\nabla \mathcal{L}_{{t_1}}&-\nabla \mathcal{L}_{{t_2}}\|_2 \leq L_1\| \omega_{{i,t_1}}  -\omega_{{i,t_2}} \|_2, \\
&\forall t_1, t_2 >0, i \in \{1,2,\dots, m\}.
\end{aligned}
\end{equation}

\noindent This also implies the following quadratic bound,
\begin{equation}
\small
\begin{aligned}
\mathcal{L}_{{t_1}} - \mathcal{L}_{{t_2}} &\leq \langle \nabla \mathcal{L}_{{t_2}}, (\omega_{{i,t_1}}  -\omega_{{i,t_2}})\rangle + \frac{L_1}{2} \|  \omega_{{i,t_1}}  -\omega_{{i,t_2}} \|_2^2 , \\
&\forall t_1, t_2 >0, \quad i \in \{1,2,\dots, m\}.
\end{aligned}
\end{equation}
\end{assumption}

\begin{assumption} \label{as2}
{\rm(Unbiased Gradient and Bounded Variance).} The stochastic gradient $g_{i,t} = \nabla \mathcal{L}(\omega_{t}, \xi_{t}) $ is an unbiased estimator of the local gradient for each client. Suppose its expectation 
\begin{equation}
\small
{\Bbb E}_{\xi_i \sim D_i}[g_{i,t}] = \nabla \mathcal{L}(\omega_{i,t}) = \nabla \mathcal{L}_{t}, \forall i \in \{1,2,\dots, m\},
\end{equation}
\noindent and its variance is bounded by $\sigma^2$: 
\begin{equation}
\small
{\Bbb E}[{\| g_{i,t} - \nabla \mathcal{L}(\omega_{i,t}) \|}_2^2]\leq \sigma^2, \forall i \in \{1,2,\dots, m\}, \sigma^2 \geq 0.
\end{equation}
\end{assumption}

\begin{assumption} \label{as3}
{\rm(Bounded Expectation of Euclidean norm of Stochastic Gradients).}The expectation of the stochastic gradient is bounded by $G$:
\begin{equation}
\small
{\Bbb E}[ \| g_{i,t} \|_2] \leq G, \forall i \in \{1,2,\dots, m\}.
\end{equation}
\end{assumption}

\begin{assumption}\label{as4}
{\rm(Lipschitz Continuity).} 
Each local embedding function is $L_2$-Lipschitz continuous, that is,
\begin{equation}
\small
\begin{aligned}
\| f_i(\phi_{i,t_1})&- f_i(\phi_{i,t_2})\| \leq L_2\| \phi_{i,t_1}  -\phi_{i,t_2} \|_2, \\
&\forall t_1, t_2 >0, i \in \{1,2,\dots, m\}.
\end{aligned}
\end{equation}
\end{assumption}

Based on the above assumptions, we present the theoretical results for the non-convex problem. The expected decrease per round is given in {Theorem \ref{th1}}. We denote $e \in \{1/2,1,2,\dots,E$\} as the local iteration, and $t$ as the global communication round. Moreover, $tE$ represents the time step before prototype aggregation, and $tE+1/2$ represents the time step between prototype aggregation and the first iteration of the current round. 

\begin{theorem} \label{th1}
{\rm (One-round deviation).} Let Assumption 1 to 4 hold. For an arbitrary client, after every communication round, we have,
{
\setlength\abovedisplayskip{-0.2cm}
\setlength\belowdisplayskip{-0.5cm}
\begin{equation} \label{eq: th1}
\small
\begin{aligned}
{\Bbb E} [\mathcal{L}_{(t+1)E+1/2}] \leq &\mathcal{L}_{tE+1/2} -\left(\eta-\frac{L_1 \eta^2}{2}\right) \sum_{e=1/2}^{E-1} \|\nabla \mathcal{L}_{tE+e} \|_2^2 \\
&+ \frac{L_1 E \eta^2}{2} \sigma^2 + \lambda L_2 \eta E G.
\end{aligned}
\end{equation}
}
\vspace{-0.4cm}
\end{theorem}

Theorem \ref{th1} indicates the deviation bound of the local objective function for an arbitrary client after each communication round. Convergence can be guaranteed when there is a certain expected one-round decrease, which can be achieved by choosing appropriate $\eta$ and $\lambda$.

\begin{corollary} \label{co1}
{\rm (Non-convex \algfont{FedProto} convergence).} 
\noindent The loss function $\mathcal{L}$ of an arbitrary client monotonously decreases in every communication round when
\begin{equation}
\small
\eta_{e'}<\frac{2( \sum_{e=1/2}^{e'} \|\nabla \mathcal{L}_{tE+e} \|_2^2- \lambda L_2 EG)}{L_1(\sum_{e=1/2}^{e'} \|\nabla \mathcal{L}_{tE+e} \|_2^2+E\sigma^2)},
\end{equation}
\noindent where $e'=1/2,1,\dots, E-1$, and
\begin{equation}
\small
\lambda_t < \frac{ \|\nabla \mathcal{L}_{tE+1/2} \|_2^2}{L_2 EG }.
\end{equation}
Thus, the loss function converges.
\end{corollary}

Corollary \ref{co1} is to ensure the expected deviation of $\mathcal{L}$ to be negative, so the loss function converges. It can guide the choice of appropriate values for the learning rate $\eta$ and the importance weight $\lambda$ to guarantee the convergence.

\begin{theorem} \label{th2}
{\rm (Non-convex convergence rate of \algfont{FedProto})}. Let Assumption 1 to 4 hold and $\Delta =\mathcal{L}_0-\mathcal{L}^* $ where $\mathcal{L}^*$ refers to the local optimum. For an arbitrary client, given any $\epsilon>0$, after 
\begin{equation}
\small
T = \frac{2\Delta}{E \epsilon (2\eta - L_1 \eta^2) - E \eta (L_1 \eta \sigma^2 + 2\lambda L_2 G)}
\end{equation} 
communication rounds of \algfont{FedProto}, we have 
\begin{equation}
\small
\frac{1}{T E}\sum_{t=0}^{T-1} \sum_{e=1/2}^{E-1}{\Bbb E}[\|\nabla \mathcal{L}_{tE+e} \|_2^2] < \epsilon,
\end{equation}
if 
{\small$$
\eta < \frac{2(\epsilon-\lambda L_2 G)}{L_1(\epsilon+\sigma^2)}  \ and\ \  \lambda <  \frac{\epsilon}{L_2 G}.
$$}
\end{theorem}

Theorem \ref{th2} provides the convergence rate, which can confine the expected L2-norm of gradients to any bound, denoted as $\epsilon$, after carefully selecting the number of communication rounds $T$ and hyperparameters including $\eta$ and $\lambda$. The smaller $\epsilon$ is, the larger $T$ is, which means that the tighter the bound is, more communication rounds is required. A detailed proof and analysis are given in Appendix B.

\vspace{-0.2cm}
\section{Discussion} \label{sec:discussion}
\vspace{-0.1cm}
In this section, we discuss the superiority of \algfont{FedProto} from three perspectives: model inference, communication efficiency, and privacy preserving.

\vspace{-0.2cm}
\subsection{Model Inference}
\vspace{-0.1cm}

Unlike many FL methods, the global model in FedProto is not a classifier but a set of class prototypes. When a new client is added to the network, one can initialize its local model with the representation layers of a pre-trained model, e.g. a ResNet18 on ImageNet, and random decision layers. Then, the local client will download the global prototypes of the classes covered in its local dataset and fine-tune the local model by minimizing the local objective. This can support new clients with novel model architectures and spend less time fine-tuning the model on heterogeneous datasets.

\vspace{-0.2cm}
\subsection{Communication Efficiency}
\vspace{-0.1cm}
Our proposed method only transmits prototypes between the server and clients. In general, the size of the prototypes is usually much smaller than the size of the model parameters. Taking MNIST as an example, the prototype size is 50 for each class, while the number of model parameters is 21,500. More details can be found in the experimental section.\looseness-1

\begin{table*}[htbp!]
\small
	\centering
	\setlength{\abovecaptionskip}{0.1cm}
	\setlength{\belowcaptionskip}{-0.15cm}
	\begin{adjustbox}{width=2.1\columnwidth,center}
	\begin{tabular}{clcccccc}
	\toprule
	\multirow{2}*{\textbf{Dataset}} & \multirow{2}*{\textbf{Method}} & \multirow{2}*{\textbf{Stdev}} & \multicolumn{3}{c} {\textbf{Test Average Acc}} & \multirow{2}*{\begin{tabular}[c]{@{}c@{}}\textbf{\# of Comm} \\ \textbf{Rounds} \end{tabular}} & \multirow{2}*{\begin{tabular}[c]{@{}c@{}}\textbf{\# of Comm} \\ \textbf{Params} ($\times 10^3$) \end{tabular}} \\
	\cmidrule{4-6}
	& & & \textbf{$ n=3 $} & \textbf{$ n=4 $} & \textbf{$ n=5 $} & & \\
	\midrule
	\multirow{8}*{\textbf{MNIST}} & {\algfont{Local}} & 2 & 94.05$\pm$2.93 & 93.35$\pm$3.26 & 92.92$\pm$3.17 & \multirow{1}*{0} & \multirow{1}*{0} \\
    & {\algfont{FeSEM} \cite{xie2020multi}} & 2 & 95.26$\pm$3.48 & \textbf{97.06$\pm$2.72} & 96.31$\pm$2.41 & \multirow{1}*{150} & \multirow{1}*{430} \\
	& {\algfont{FedProx} \cite{li2018federated}} & 2 & 96.26$\pm$2.89 & 96.40$\pm$3.33 & 95.65$\pm$3.38 & \multirow{1}*{110} & \multirow{1}*{430} \\
	& {\algfont{FedPer} \cite{arivazhagan2019federated}} & 2 & 95.57$\pm$2.96 & 96.44$\pm$2.62 & 95.55$\pm$3.13 & \multirow{1}*{100} & \multirow{1}*{106} \\
	& {\algfont{FedAvg} \cite{mcmahan2016communication}} & 2 & 95.04$\pm$6.48 & 94.32$\pm$4.89 & 93.22$\pm$4.39 & \multirow{1}*{150} & \multirow{1}*{430} \\
	& {\algfont{FedRep} \cite{collins2021exploiting}} & 2 & 94.96$\pm$2.78 & 95.18$\pm$3.80 & 94.94$\pm$2.81 & \multirow{1}*{100} & \multirow{1}*{110} \\
% 	& {\algfont{FedDF} \cite{lin2020ensemble}} & 2 & $\pm$ & $\pm$ &  & \multirow{1}*{} & \multirow{1}*{} \\
	& {\textbf{\algfont{FedProto}}} & 2 & \textbf{97.13$\pm$0.30} & 96.80$\pm$0.41 & \textbf{96.70$\pm$0.29} & \multirow{1}*{{100}} & \multirow{1}*{\textbf{4}} \\
	& {\textbf{\algfont{FedProto-mh}}} & 2 & 97.07$\pm$0.50 & 96.65$\pm$0.31 & 96.22$\pm$0.36 & \multirow{1}*{100} & \multirow{1}*{\textbf{4}} \\
	\midrule
	
	\multirow{8}*{\textbf{FEMNIST}} & \multirow{1}*{\algfont{Local}} & 1 & 92.50$\pm$10.42	 & 91.16$\pm$5.64 & 87.91$\pm$8.44 & \multirow{1}*{0} & \multirow{1}*{0} \\
    & \multirow{1}*{\algfont{FeSEM} \cite{xie2020multi}} & 1 & 93.39$\pm$6.75 & 91.06$\pm$6.43 & 89.61$\pm$7.89 & \multirow{1}*{200} & \multirow{1}*{16,000} \\
	& \multirow{1}*{\algfont{FedProx} \cite{li2018federated}} & 1 & 94.53$\pm$5.33 & 90.71$\pm$6.24 & 91.33$\pm$7.32 & \multirow{1}*{300} & \multirow{1}*{16,000} \\
	& \multirow{1}*{\algfont{FedPer} \cite{arivazhagan2019federated}} & 1 & 93.47$\pm$5.44 & 90.22$\pm$7.63 & 87.73$\pm$9.64 & \multirow{1}*{250} & \multirow{1}*{102} \\
	& \multirow{1}*{\algfont{FedAvg} \cite{mcmahan2016communication}} & 1 & 94.50$\pm$5.29 & 91.39$\pm$5.23 & 90.95$\pm$7.22 & \multirow{1}*{300} & \multirow{1}*{16,000} \\
	& \multirow{1}*{\algfont{FedRep} \cite{collins2021exploiting}} & 1 & 93.36$\pm$5.34 & 91.41$\pm$5.89 & 	89.98$\pm$6.88  & \multirow{1}*{200} & \multirow{1}*{102} \\
% 	& {\algfont{FedDF} \cite{lin2020ensemble}} & 1 & $\pm$ & $\pm$ &  & \multirow{1}*{} & \multirow{1}*{} \\
	& \multirow{1}*{\textbf{\algfont{FedProto}}} & 1 & 96.82$\pm$1.75 & \textbf{94.93$\pm$1.61} & 93.67$\pm$2.23 & \multirow{1}*{{120}} & \multirow{1}*{\textbf{4}} \\
	& \multirow{1}*{\textbf{\algfont{FedProto-mh}}} & 1 & \textbf{97.10$\pm$1.63} & 94.83$\pm$1.60 & \textbf{93.76$\pm$2.30} & \multirow{1}*{{120}} & \multirow{1}*{\textbf{4}} \\
	\midrule
	
	\multirow{8}*{\textbf{CIFAR10}} & \multirow{1}*{\algfont{Local}} & 1 & 79.72$\pm$9.45 & 67.62$\pm$7.15 & 58.64$\pm$6.57 & \multirow{1}*{0} & \multirow{1}*{0} \\
	& \multirow{1}*{\algfont{FeSEM} \cite{xie2020multi}} & 1 & 80.19$\pm$3.31 & 76.40$\pm$3.23 & 74.17$\pm$3.51 & \multirow{1}*{120} & \multirow{1}*{235,000} \\
	& \multirow{1}*{\algfont{FedProx} \cite{li2018federated}} & 1 & 83.25$\pm$2.44 & 79.20$\pm$1.31 & 76.19$\pm$2.23 & \multirow{1}*{150} & \multirow{1}*{235,000} \\
	& \multirow{1}*{\algfont{FedPer} \cite{arivazhagan2019federated}} & 1 & 84.38$\pm$4.58 & 78.73$\pm$4.59 & 76.21$\pm$4.27 & \multirow{1}*{130} & \multirow{1}*{225,000} \\
	& \multirow{1}*{\algfont{FedAvg} \cite{mcmahan2016communication}} & 1 & 81.72$\pm$2.77 & 76.77$\pm$2.37 & 75.74$\pm$2.61  & \multirow{1}*{150} & \multirow{1}*{235,000} \\
	& \multirow{1}*{\algfont{FedRep} \cite{collins2021exploiting}} & 1 & 81.44$\pm$10.48	 & 76.93$\pm$7.46	 & 73.36$\pm$7.04 & \multirow{1}*{110} & \multirow{1}*{225,000} \\
% 	& {\algfont{FedDF} \cite{lin2020ensemble}} & 1 & $\pm$ & $\pm$ &  & \multirow{1}*{} & \multirow{1}*{} \\
	& \multirow{1}*{\textbf{\algfont{FedProto}}} & 1 &  \textbf{84.49$\pm$1.97} & 79.12$\pm$2.03 & \textbf{77.08$\pm$1.98} & \multirow{1}*{{110}} & \multirow{1}*{\textbf{41}} \\
	& \multirow{1}*{\textbf{\algfont{FedProto-mh}}} & 1 & 83.63$\pm$1.60 & \textbf{79.49$\pm$1.78} & 76.94$\pm$1.33  & \multirow{1}*{{110}} & \multirow{1}*{\textbf{41}} \\
	\bottomrule
	\end{tabular}
	\end{adjustbox}
	\caption{Comparison of FL methods on three benchmark datasets with non-IID split over clients. The best results are in \textbf{bold}. It appears that \algfont{FedProto}, compared to baselines, achieves higher accuracy while using much fewer communicated parameters.} \vspace{-0.5cm}
	\label{perf}
\end{table*}

\vspace{-0.2cm}
\subsection{Privacy Preserving}
\vspace{-0.1cm}

The proposed \algfont{FedProto} requires the exchange of prototypes rather than model parameters between the server and the clients. This property brings benefits to FL in terms of privacy preserving. First, prototypes naturally protect the data privacy, because they are 1D-vectors generated by averaging the low-dimension representations of samples from the same class, which is an irreversible process. Second, attackers cannot reconstruct raw data from prototypes without the access to local models. Moreover, \algfont{FedProto} can be integrated with various privacy-preserving techniques to further enhance the reliability of the system.

\vspace{-0.15cm}
\section{Experiments} \label{sec:experiments}
\vspace{-0.15cm}

\subsection{Training Setups}
\vspace{-0.15cm}

\paragraph{Datasets and local models} 
We implement the typical federated setting where each client owns its local data and transmits/receives information to/from the central server. We use three popular benchmark datasets: MNIST \cite{lecun1998mnist}, FEMNIST \cite{caldas2018leaf} and CIFAR10 \cite{krizhevsky2009learning}. We consider a multi-layer CNN which consists of 2 convolutional layers then 2 fully connected layers for both MNIST and FEMNIST, and ResNet18 \cite{he2016deep} for CIFAR10.
\vspace{-0.2cm}

\paragraph{Local tasks} 
Each client learns a supervised learning task. In particular, to illustrate the local task, we borrow the concept of $n$-way $k$-shot from few-shot learning where $n$ controls the number of classes and $k$ controls the number of training instances per class. To mimic the heterogeneous scenario, we randomly change the value of $n$ and $k$ in different clients. We define an average value for $n$ and $k$, and then add a random noise to each user's $n$ as well as $k$. The purpose of the variance of $n$ is to control the heterogeneity of the class space, while the purpose of the variance of $k$ is to control the imbalance in data size.
\vspace{-0.2cm}

\paragraph{Baselines of FL} We study the performance of \algfont{FedProto} under both the statistical and model heterogeneous settings (\algfont{FedProto-mh}) and make comparisons with baselines, including \algfont{Local} where an individual model is trained for each client without any communication with others, \algfont{FedAvg} \cite{mcmahan2016communication}, \algfont{FedProx} \cite{li2018federated}, \algfont{FeSEM} \cite{xie2020multi}, \algfont{FedPer} \cite{arivazhagan2019federated}, and \algfont{FedRep} \cite{collins2021exploiting}.
% , and \algfont{FedDF} \cite{lin2020ensemble}.

\vspace{-0.1cm}

\begin{figure*}[!htbp]
	\centering
	\setlength{\abovecaptionskip}{0.cm}
	\setlength{\belowcaptionskip}{-0.5cm}
	\subfigure[\algfont{FedProto}]{
		\begin{minipage}[b]{0.235\textwidth}
			\includegraphics[width=1\textwidth]{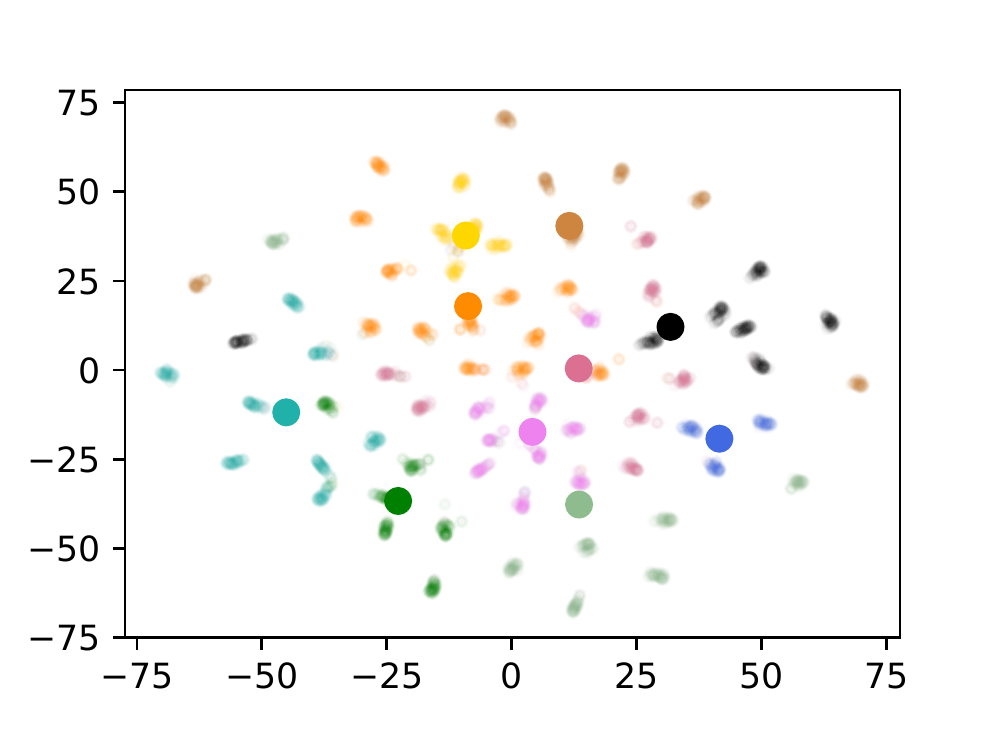}
		\end{minipage}
		\label{protos_fedproto}
	}	
	\subfigure[\algfont{FedAvg}]{
		\begin{minipage}[b]{0.235\textwidth}
			\includegraphics[width=1\textwidth]{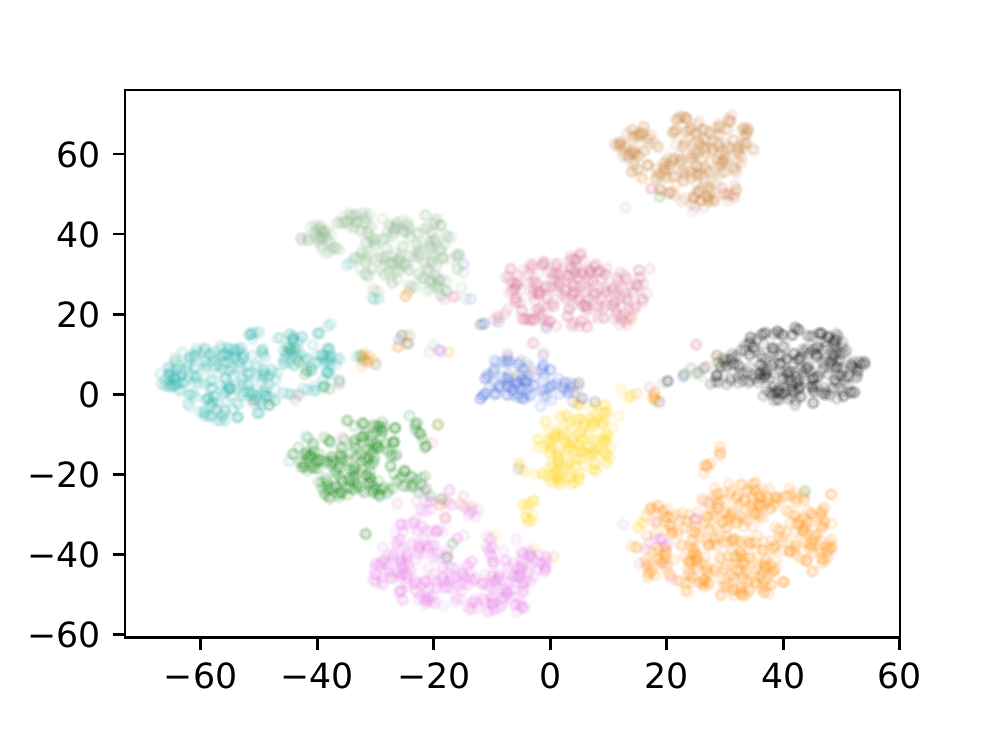}
		\end{minipage}
		\label{protos_fedavg}
	} 
	\subfigure[\algfont{FeSEM}]{
		\begin{minipage}[b]{0.235\textwidth}
			\includegraphics[width=1\textwidth]{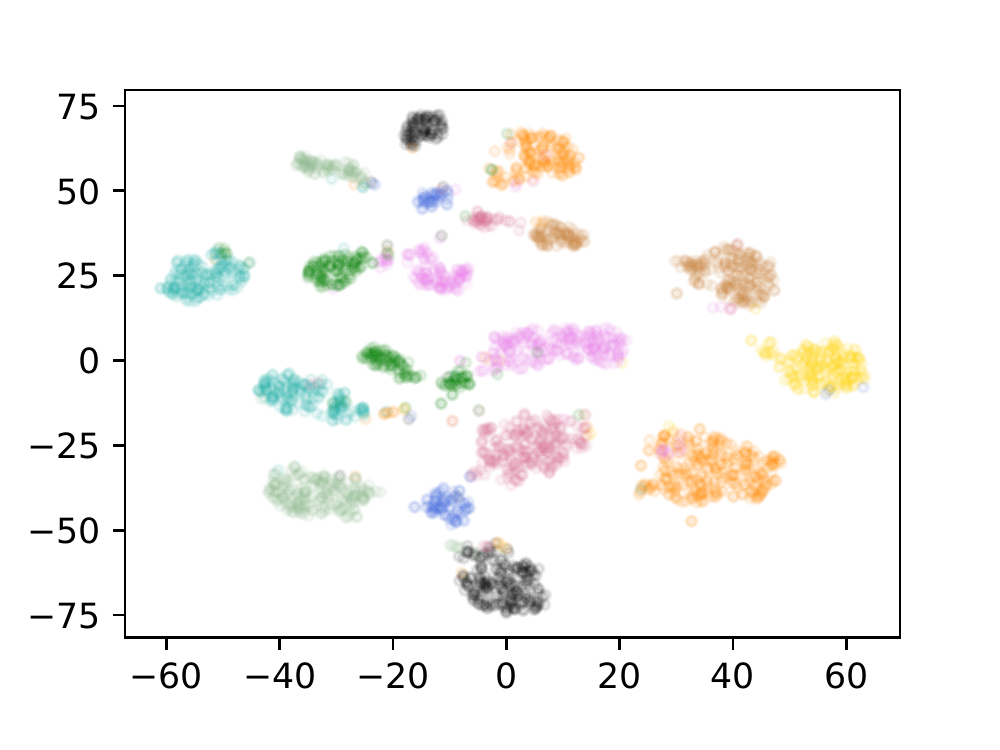}
		\end{minipage}
		\label{protos_fedsem}
	}
	\subfigure[\algfont{FedPer}]{
		\begin{minipage}[b]{0.235\textwidth}
			\includegraphics[width=1\textwidth]{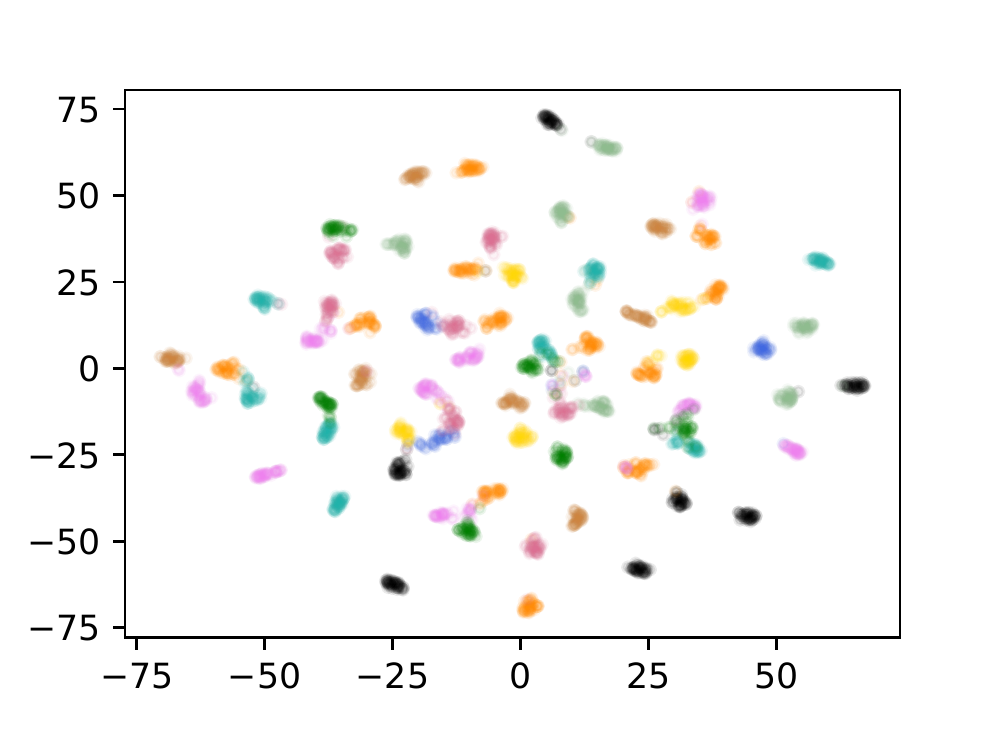}
		\end{minipage}
		\label{protos_fedper}
	}
	\caption{t-SNE visualization of the samples and/or prototypes produced by \algfont{FedProto} and other FL methods. We consider 20 clients for MNIST. The average number of classes per client is $n=3$. (a) \algfont{FedProto}: Samples within the same class have multiple centers with each center representing local prototype of a client. The global prototype is the central point of the class samples. (b) \algfont{FedAvg}: Samples within the same class cluster in the same area. (c) \algfont{FeSEM}: Samples within the same class gather in several clusters according to the algorithm parameters. (d) \algfont{FedPer}: Each cluster indicates one class in a specific client.}
	\label{vis}
\end{figure*}

\paragraph{Implementation Details}
We implement \algfont{FedProto} and the baseline methods in PyTorch. We use 20 clients for all datasets and all clients are sampled in each communication round. The average size of each class in each client is set to be 100. For MNIST and FEMNIST dataset, our initial set of hyperparameters was taken directly from the default set of hyperparamters in \cite{mcmahan2016communication}. For CIFAR10, ResNet18 pre-trained on ImageNet \cite{krizhevsky2017imagenet} is used as the initial model. The initial average test accuracy of the pre-trained network on CIFAR10 is 27.55$\%$. A detailed setup including the choice of hyperparameters is given in Appendix A.
\vspace{-0.25cm}

\subsection{Performance in Non-IID Federated Setting}
\vspace{-0.15cm}
We compare \algfont{FedProto} with other baseline methods that are either classical FL methods or FL methods with an emphasis on statistical heterogeneity. All methods are adapted to fit this heterogeneous setting.
\vspace{-0.2cm}

\paragraph{Statistical heterogeneity simulations} In our setting, we assume that all clients perform learning tasks with heterogeneous statistical distributions. In order to simulate different levels of heterogeneity, we fix the standard deviation to be 1 or 2, aiming to create heterogeneity in both class spaces and data sizes, which is common in real-world scenarios. 
\vspace{-0.2cm}

\paragraph{Model heterogeneity simulations} For the model heterogeneous setting, we consider minor differences in model architectures across clients. In MNIST and FEMNIST, the number of output channels in the convolutional layers is set to either 18, 20 or 22, while in CIFAR10, the stride of convolutional layers is set differently across different clients. This kind of model heterogeneity brings about challenges for model parameter averaging because the parameters in different clients are not always the same size.

The average test accuracy over all clients is reported in Table \ref{perf}. It can be seen that \algfont{FedProto} achieves the highest accuracy and the least variance in most cases, ensuring uniformity among heterogeneous clients.
\vspace{-0.15cm}

\paragraph{Communication efficiency} Communication costs have always been posed as a challenge in FL, considering several limitations in existing communication channels. Therefore, we also report the number of communication rounds required for convergence and the number of parameters communicated per round in Table \ref{perf}. It can be seen that the number of parameters communicated per round in \algfont{FedProto} is much lower than in the case of \algfont{FedAvg}. Furthermore, \algfont{FedProto} requires the fewest communication rounds for the local optimization. This suggests that when the heterogeneity level is high across the clients, sharing more parameters does not always lead to better results. It is more important to identify which part to share in order to benefit the current system to a great extent. More performance results are shown in Appendix A.
\vspace{-0.15cm}

\paragraph{Visualization of prototypes achieved by \algfont{FedProto}}
We visualize the samples in MNIST test set by t-SNE~\cite{van2008visualizing}. In Figure \ref{protos_fedproto}, small points in different colors represent samples in different classes, with large points representing corresponding global prototypes. In Figure \ref{protos_fedavg}, \ref{protos_fedsem} and \ref{protos_fedper}, the points in different colors refer to the representations of samples belonging to different classes. Better generalization means that there are more samples within the same class cluster in the same area, which can be achieved in a centralized setting, while better personalization means that it is easier to determine to which client the samples belong. It can be seen that samples within the same class but from various clients are close but separable in \algfont{FedProto}. This indicates that \algfont{FedProto} is more successful in achieving the balance between generalization and personalization, while other methods lacks either the generalization or the personalization ability.\looseness-1
\vspace{-0.2cm}

\paragraph{Scalability of \algfont{FedProto} on varying number of samples}
\begin{figure}[htbp!]
	\centering
	\setlength{\abovecaptionskip}{-0.2cm}
	\setlength{\belowcaptionskip}{-0.6cm}
	\includegraphics[width=0.35\textwidth]{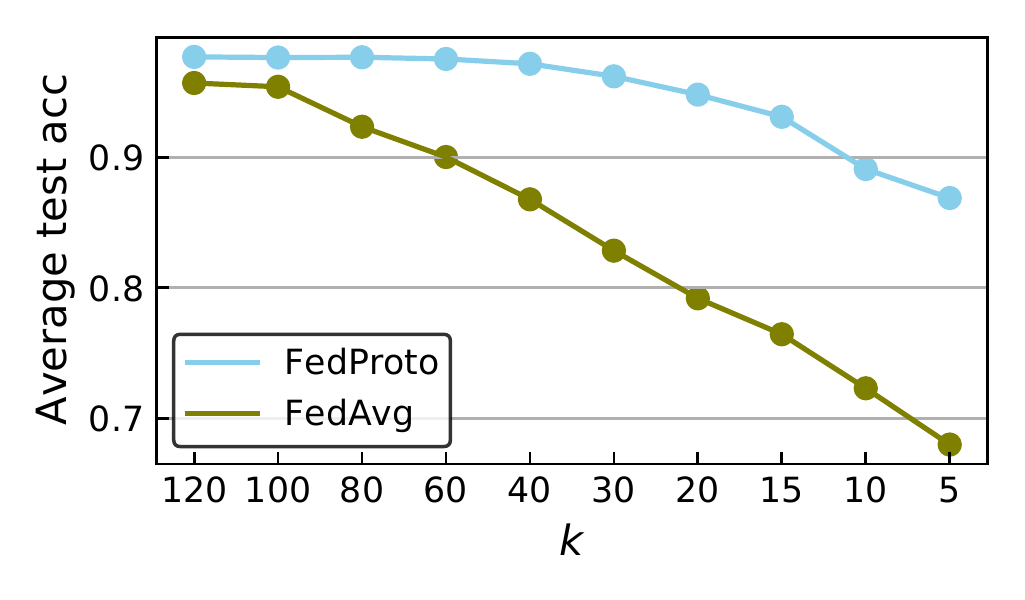}
	\caption{Average test accuracy of \algfont{FedProto} and \algfont{FedAvg} on MNIST with varying number of samples in each class.}
	\label{k_acc}
\end{figure}

Figure \ref{k_acc} shows that \algfont{FedProto} can scale to scenarios with fewer samples available on clients. The test accuracy consistently decreases when there are fewer samples for training, but \algfont{FedProto} drops more slowly than \algfont{FedAvg} as a result of its adaptability and scalability on various data sizes. %\looseness-1
\vspace{-0.2cm}

\paragraph{\algfont{FedProto} under varying $\lambda$}
Figure \ref{lambda} shows the varying performance under different values of $\lambda$ in Eq. (\ref{eq:lambda}). We tried a set of values selected from $[0,4]$ and reported the average test accuracy and proto distance loss with $n$=3, $k$=100 in FEMNIST dataset. The best value of $\lambda$ is $1$ in this scenario. As $\lambda$ increases, the proto distance loss (regularization term) decreases, while the average test accuracy experiences a sharp rise from $\lambda$=0 to $\lambda$=1 before a drop in the number of 6$\%$, demonstrating the efficacy of prototype aggregation. \looseness-1
\vspace{-0.1cm}

\begin{figure}[htbp!]
	\centering
	\setlength{\abovecaptionskip}{-0.2cm}
	\setlength{\belowcaptionskip}{-0.5cm}
	\includegraphics[width=0.4\textwidth]{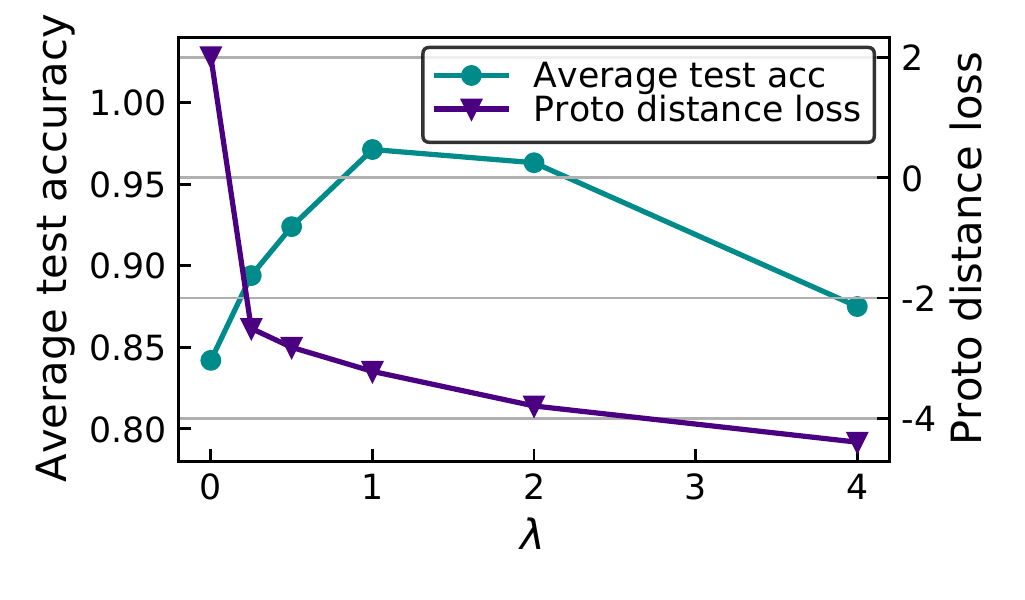}
	\caption{Average test accuracy on FEMNIST under varying importance weight $\lambda$.}
	\label{lambda}
\end{figure}
\vspace{-0.2cm}

\section{Conclusion} \label{sec:conclusion}
In this paper, we propose a novel prototype aggregation-based FL method to tackle challenging FL scenarios with heterogeneous input/output spaces, data distributions, and model architectures. The proposed method collaboratively trains intelligent models by exchanging prototypes rather than gradients, which offers new insights for designing prototype-based FL. The effectiveness of the proposed method has been comprehensively analyzed from both theoretical and experimental perspectives.  \looseness-1 

\clearpage
\bibliography{aaai22_fedproto}

\clearpage

% \title{Supplementary Materials}
% \maketitle

\onecolumn
\noindent We present the related supplements in following sections.

\section{Experimental Details and Extra Results} \label{appendix_a}
\subsection{Experimental Details}

Local clients are trained by SGD optimizer, with a learning rate of $0.01$ and momentum of $0.5$.
Regarding the crucial hyperparameter $\lambda$, we tune the best $\lambda$ from a limited candidate set by \textit{grid search}. The best $\lambda$ values for MNIST, FEMNIST and CIFAR10 are $1$, $1$ and $0.1$, respectively. The number of local epochs and local batch size are set to be 1 and 8, respectively, for all datasets. The heterogeneity level of clients is controlled by the standard deviation of $ n $. The higher this is, the more heterogeneous the clients are.

\subsection{Extra Results}
The complete experimental results show the performance of \algfont{FedProto} and \algfont{FedProto-mh} on three benchmark datasets MNIST, FEMNIST, and CIFAR10. Compared with existing FL methods, \algfont{FedProto} yields higher test accuracy while resulting in lower communication costs under different heterogeneous settings. Additionally, it can be used in model heterogeneous scenarios and achieves performance similar to that in homogeneous scenarios.

\vspace{0.2cm}

\noindent For MNIST, we evaluate local test sets and report the evaluation results in Table \ref{acc_mnist}. It appears that \algfont{FedProto} achieves strong performance with low communication cost. The local average test accuracy of \algfont{FedProto} is greater than for the \algfont{FeSEM}, \algfont{FedProx}, \algfont{FedPer} and \algfont{FedAvg} algorithms in all the settings. 

\vspace{0.2cm}

\noindent For FEMNIST, the evaluation results are reported in Table \ref{acc_femnist}. We consider the standard deviation of $ n $ to be 1 and 2. The results show that, for \algfont{FedProto}, the variance of the accuracy across clients is much smaller than for other FL methods, thus ensuring uniformity among heterogeneous clients. \algfont{FedProto} allows us to better utilize the local FEMNIST dataset distribution while using around $ 0.025\% $ of the total parameters communicated.

\vspace{0.2cm}

\noindent For CIFAR10, as can be seen in Table \ref{acc_cifar10}, \algfont{FedProto} converges faster in the presence of heterogeneity in most cases. In \algfont{FedProto} and \algfont{FedProto-mh}, the number of parameters communicated per round is much lower than the baseline methods, meaning greatly reduced communication costs.

\begin{table*}[htbp!]
	\centering
	\begin{tabular}{cccccccc}
	\toprule
	\multirow{2}*{\textbf{Dataset}} & \multirow{2}*{\textbf{Method}} & \multirow{2}*{\begin{tabular}[c]{@{}c@{}}\textbf{Stdev} \\ \textbf{of} $n$ \end{tabular}} & \multicolumn{3}{c} {\textbf{Test Average Acc}} & \multirow{2}*{\begin{tabular}[c]{@{}c@{}}\textbf{\# of Comm} \\ \textbf{Rounds} \end{tabular}} & \multirow{2}*{\begin{tabular}[c]{@{}c@{}}\textbf{\# of Comm} \\ \textbf{Params} ($\times 10^3$) \end{tabular}} \\
	\cmidrule{4-6}
	& & & \textbf{$ n=3 $} & \textbf{$ n=4 $} & \textbf{$ n=5 $} & & \\
	\midrule
	\multirow{16}*{\textbf{MNIST}} & \multirow{2}*{\algfont{Local}} & 2 & 94.05$\pm$2.93 & 93.35$\pm$3.26 & 	92.92$\pm$3.17 & \multirow{2}*{0} & \multirow{2}*{0} \\
	& & 3 & 93.44$\pm$3.57 & 94.24$\pm$2.49	 & 93.97$\pm$2.97 \\
	\cmidrule{4-8}
	& \multirow{2}*{\algfont{FeSEM}} & 2 & 95.26$\pm$3.48 & \textbf{97.06$\pm$2.72} & 96.31$\pm$2.41 & \multirow{2}*{150} & \multirow{2}*{430} \\
	& & 3 & 96.40$\pm$3.35 & 95.82$\pm$3.94 & 95.98$\pm$2.46 \\
	\cmidrule{4-8}
	& \multirow{2}*{\algfont{FedProx}} & 2 & 96.26$\pm$2.89 & 96.40$\pm$3.33 & 95.65$\pm$3.38 & \multirow{2}*{110} & \multirow{2}*{430} \\
	& & 3 & 96.65$\pm$3.28 & 95.25$\pm$3.73 & 95.34$\pm$2.85 \\
	\cmidrule{4-8}
	& \multirow{2}*{\algfont{FedPer}} & 2 & 95.57$\pm$2.96 & 96.44$\pm$2.62 & 95.55$\pm$3.13 & \multirow{2}*{100} & \multirow{2}*{106} \\
	& & 3 & 96.57$\pm$2.65 & 95.93$\pm$2.76 & 96.07$\pm$2.80 \\
	\cmidrule{4-8}
	& \multirow{2}*{\algfont{FedAvg}} & 2 & 91.40$\pm$6.48 & 94.32$\pm$4.89 & 93.22$\pm$4.39 & \multirow{2}*{150} & \multirow{2}*{430} \\
	& & 3 & 94.57$\pm$4.91 & 91.99$\pm$6.89 & 92.19$\pm$3.97 \\
	\cmidrule{4-8}
	& \multirow{2}*{\algfont{FedRep}} & 2 & 94.96$\pm$2.78 & 95.18$\pm$3.80	 & 94.94$\pm$2.81 & \multirow{2}*{100} & \multirow{2}*{110} \\
	& & 3 & 95.01$\pm$3.92 & 95.55$\pm$2.79 & 95.38$\pm$2.97 \\
	\cmidrule{4-8}
	& \multirow{2}*{\textbf{\algfont{FedProto}}} & 2 & \textbf{97.13$\pm$0.30} & 96.80$\pm$0.41 & \textbf{96.70$\pm$0.29} & \multirow{2}*{\textbf{100}} & \multirow{2}*{\textbf{4}} \\
	& & 3 & \textbf{96.71$\pm$0.43} & \textbf{96.87$\pm$0.28} & \textbf{96.47$\pm$0.23} \\
	\cmidrule{4-8}
	& \multirow{2}*{\textbf{\algfont{FedProto-mh}}} & 2 & 97.07$\pm$0.50 & 96.65$\pm$0.31 & 96.22$\pm$0.36 & \multirow{2}*{\textbf{100}} & \multirow{2}*{\textbf{4}} \\
	& & 3 & 96.48$\pm$0.43 & 96.84$\pm$0.33 & 95.56$\pm$0.31 \\
	\bottomrule
	\end{tabular}
	\caption{Comparison of FL methods on MNIST with non-IID split over clients. Best results are in \textbf{bold}. It appears that \algfont{FedProto}, compared to the baseline methods, achieves higher accuracy while using much less communicated parameters.}
	\label{acc_mnist}
\end{table*}

\begin{table*}[htbp!]
	\centering
	\begin{tabular}{cccccccc}
	\toprule
	\multirow{2}*{\textbf{Dataset}} & \multirow{2}*{\textbf{Method}} & \multirow{2}*{\begin{tabular}[c]{@{}c@{}}\textbf{Stdev} \\ \textbf{of} $n$ \end{tabular}} & \multicolumn{3}{c} {\textbf{Test Average Acc}} & \multirow{2}*{\begin{tabular}[c]{@{}c@{}}\textbf{\# of Comm} \\ \textbf{Rounds} \end{tabular}} & \multirow{2}*{\begin{tabular}[c]{@{}c@{}}\textbf{\# of Comm} \\ \textbf{Params} ($\times 10^3$) \end{tabular}} \\
	\cmidrule{4-6}
	& & & \textbf{$ n=3 $} & \textbf{$ n=4 $} & \textbf{$ n=5 $} & & \\
	\midrule
	\multirow{16}*{\textbf{FEMNIST}} & \multirow{2}*{\algfont{Local}} & 1 & 92.50$\pm$10.42 & 91.16$\pm$5.64 & 87.91$\pm$8.44 & \multirow{2}*{0} & \multirow{2}*{0} \\
	& & 2 & 92.11$\pm$6.02 & 90.34$\pm$6.42 & 89.70$\pm$6.33 \\
	\cmidrule{4-8}
	& \multirow{2}*{\algfont{FeSEM}} & 1 & 93.39$\pm$6.75 & 91.06$\pm$6.43 & 89.61$\pm$7.89 & \multirow{2}*{200} & \multirow{2}*{16,000} \\
	& & 2 & 94.19$\pm$4.90 & 93.52$\pm$4.47 & 90.77$\pm$6.70 \\
	\cmidrule{4-8}
	& \multirow{2}*{\algfont{FedProx}} & 1 & 94.53$\pm$5.33 & 90.71$\pm$6.24 & 91.33$\pm$7.32 & \multirow{2}*{300} & \multirow{2}*{16,000} \\
	& & 2 & 93.49$\pm$5.30 & 93.74$\pm$5.02 & 89.49$\pm$6.74 \\
	\cmidrule{4-8}
	& \multirow{2}*{\algfont{FedPer}} & 1 & 93.47$\pm$5.44 & 90.22$\pm$7.63 & 87.73$\pm$9.64 & \multirow{2}*{250} & \multirow{2}*{102} \\
	& & 2 & 92.27$\pm$6.16 & 91.99$\pm$6.33 & 87.54$\pm$8.14 \\
	\cmidrule{4-8}
	& \multirow{2}*{\algfont{FedAvg}} & 1 & 94.50$\pm$5.29 & 91.39$\pm$5.23 & 90.95$\pm$7.22 & \multirow{2}*{300} & \multirow{2}*{16,000} \\
	& & 2 & 94.13$\pm$4.92 & 93.02$\pm$5.77 & 89.80$\pm$6.94 \\
	\cmidrule{4-8}
	& \multirow{2}*{\algfont{FedRep}} & 1 & 93.36$\pm$5.34 & 91.41$\pm$5.89	 & 89.98$\pm$6.88 & \multirow{2}*{200} & \multirow{2}*{102} \\
	& & 2 & 92.28$\pm$5.40 & 91.56$\pm$7.02 & 88.23$\pm$6.97 \\
	\cmidrule{4-8}
	& \multirow{2}*{\textbf{\algfont{FedProto}}} & 1 & 96.82$\pm$1.75 & \textbf{94.93$\pm$1.61} & 93.67$\pm$2.23 & \multirow{2}*{\textbf{120}} & \multirow{2}*{\textbf{4}} \\
	& & 2 & 94.93$\pm$1.29 & 94.69$\pm$1.50 & \textbf{93.03$\pm$2.50} \\
	\cmidrule{4-8}
	& \multirow{2}*{\textbf{\algfont{FedProto-mh}}} & 1 & \textbf{97.10$\pm$1.63} & 94.83$\pm$1.60 & \textbf{93.76$\pm$2.30} & \multirow{2}*{\textbf{120}} & \multirow{2}*{\textbf{4}} \\
	& & 2 & \textbf{95.33$\pm$1.30} & \textbf{94.98$\pm$1.69} & 92.94$\pm$2.34 \\
	\midrule
	\end{tabular}
	\caption{Comparison of FL methods on FEMNIST with non-IID split over clients. Best results are in \textbf{bold}. It appears that \algfont{FedProto}, compared to the baseline methods, achieves higher accuracy while using much fewer communicated parameters.}
	\label{acc_femnist}
\end{table*}

\begin{table*}[htbp!]
	\centering
	\begin{tabular}{cccccccc}
	\toprule
	\multirow{2}*{\textbf{Dataset}} & \multirow{2}*{\textbf{Method}} & \multirow{2}*{\begin{tabular}[c]{@{}c@{}}\textbf{Stdev} \\ \textbf{of} $n$ \end{tabular}} & \multicolumn{3}{c} {\textbf{Test Average Acc}} & \multirow{2}*{\begin{tabular}[c]{@{}c@{}}\textbf{\# of Comm} \\ \textbf{Rounds} \end{tabular}} & 
	\multirow{2}*{\begin{tabular}[c]{@{}c@{}}\textbf{\# of Comm} \\ \textbf{Params} ($\times 10^4$) \end{tabular}} \\
	\cmidrule{4-6}
	& & & \textbf{$ n=3 $} & \textbf{$ n=4 $} & \textbf{$ n=5 $} & & \\
	\midrule
	\multirow{16}*{\textbf{CIFAR10}} & \multirow{2}*{\algfont{Local}} & 1 & 79.72$\pm$9.45 & 67.62$\pm$7.15 & 58.64$\pm$6.57  & \multirow{2}*{0} & \multirow{2}*{0} \\
	& & 2  & 68.15$\pm$9.88	 & 61.03$\pm$11.83 & 58.81$\pm$12.90\\
	\cmidrule{4-8}
	& \multirow{2}*{\algfont{FeSEM}} & 1 & 80.19$\pm$3.31 & 76.40$\pm$3.23 & 74.17$\pm$3.51  & \multirow{2}*{120} & \multirow{2}*{2.35$\times10^4$} \\
	& & 2  & 76.12$\pm$4.15 & 72.11$\pm$3.48 & 70.89$\pm$3.39\\
	\cmidrule{4-8}
	& \multirow{2}*{\algfont{FedProx}} & 1 &  83.25$\pm$2.44 & 79.20$\pm$1.31 & 76.19$\pm$2.23  & \multirow{2}*{150} & \multirow{2}*{2.35$\times10^4$} \\
	& & 2 & 79.83$\pm$2.35 & 72.56$\pm$1.90 & 71.39$\pm$2.36 \\
	\cmidrule{4-8}
	& \multirow{2}*{\algfont{FedPer}} & 1 &  84.38$\pm$4.58 & 78.73$\pm$4.59 & 76.21$\pm$4.27  & \multirow{2}*{130} & \multirow{2}*{2.25$\times10^4$} \\
    & & 2 & \textbf{84.51$\pm$4.39} & 73.31$\pm$4.76 & 72,43$\pm$4.55 \\
	\cmidrule{4-8}
	& \multirow{2}*{\algfont{FedAvg}} & 1 &  81.72$\pm$2.77 & 76.77$\pm$2.37 & 75.74$\pm$2.61  & \multirow{2}*{150} & \multirow{2}*{2.35$\times10^4$} \\
    & & 2 & 78.99$\pm$2.34 & 72.73$\pm$2.58 & 70.93$\pm$2.82 \\
	\cmidrule{4-8}
	& \multirow{2}*{\algfont{FedRep}} & 1 &  81.44$\pm$10.48 & 76.93$\pm$7.46	 & 73.36$\pm$7.04  & \multirow{2}*{110} & \multirow{2}*{2.25$\times10^4$} \\
    & & 2 & 76.70$\pm$11.79 & 73.54$\pm$11.42 & 70.30$\pm$8.00 \\
	\cmidrule{4-8}
	& \multirow{2}*{\textbf{\algfont{FedProto}}} & 1 &  \textbf{84.49$\pm$1.97} & 79.12$\pm$2.03 & \textbf{77.08$\pm$1.98}  & \multirow{2}*{\textbf{110}} & \multirow{2}*{\textbf{4.10}} \\
    & & 2 & 81.75$\pm$1.39 & 74.98$\pm$1.61 & 71.17$\pm$1.29\\
	\cmidrule{4-8}
	& \multirow{2}*{\textbf{\algfont{FedProto-mh}}} & 1 &  83.63$\pm$1.60 & \textbf{79.49$\pm$1.78} & 76.94$\pm$1.33  & \multirow{2}*{\textbf{110}} & \multirow{2}*{\textbf{4.10}} \\
	& & 2 & 79.90$\pm$1.08 & \textbf{75.78$\pm$1.05} & \textbf{72.67$\pm$1.09}\\
	\bottomrule
	\end{tabular}
	\caption{Comparison of FL methods on CIFAR10 with non-IID split over clients. Best results are in \textbf{bold}. It appears that \algfont{FedProto}, compared to the baseline methods, achieves higher accuracy while using much fewer communicated parameters.}
	\label{acc_cifar10}
\end{table*}

\clearpage

\section{Convergence Analysis for \algfont{FedProto}} \label{appendix_b}
\subsection{Additional Notation}
Here, additional variables are introduced to better represent the process of local model update.
Let $f_i(\phi_i): {\Bbb R}^{d_x} \rightarrow {\Bbb R}^{d_c}$ be the embedding function of the $i$-th client, which can be different regarding to different clients. $d_x$ and $d_c$ represent the dimension of the input $x$ and the prototype, respectively. They should be the same for all clients.
$g_i(\nu_i): {\Bbb R}^{d_c} \rightarrow {\Bbb R}^{d_y}$ is the decision function for all clients, in which $d_y$ represents the dimension of output $y$. So the labelling function can be written as $\mathcal{F}_i(\phi_i, \nu_i) = g_i(\nu_i) \circ f_i(\phi_i)$, and sometimes we use $\omega_i$ to represent $(\phi_i, \nu_i)$ for short.
In the theoretical analysis, we omit the label $(j)$ of prototype $C^{(j)}$
for convenience, which does not affect the proof. We also use $q_i$ to represent the weight of the prototype for $i$-th client, and $p_i$ to represent the weight of the loss function for the $i$-th client for short.

\noindent Therefore, the local loss function of client $i$ can be written as:
\setcounter{equation}{0}
\begin{equation} \label{eq:local_re}
\mathcal{L}(\phi_i, \nu_i; x,y) = \mathcal{L}_S (\mathcal{F}_i(\phi_i, \nu_i;x),y)+ \lambda \|f_i(\phi_i;x) - \bar{C} \|_2^2, 
\end{equation}

\noindent in which the global prototype 
\begin{equation} \label{eq:global_proto}
\bar{C} = \sum_{i=1}^m q_i C_i %= \sum_{i=1}^m q_i \bar{f}_i(\phi_i;x_i)
\end{equation}

\noindent with 
\begin{equation}
\sum_{i=1}^m q_i = \sum_{i=1}^m \frac{|D_i|}{N}=1
\end{equation}
\noindent and 
\begin{equation} \label{eq:local_proto}
C_i = \frac{1}{|D_i|} \sum_{(x,y) \in D_i} f_i(\phi_i;x), 
\end{equation}

\noindent and it is a constant in $\mathcal{L}$, changing $\mathcal{L}$ every communication round, which makes the convergence analysis complex.

\noindent As for the iteration notation system, we use $t$ to represent the communication round, $e \in \{1/2,1,2,\dots,E$\} to represent the local iterations. There are $E$ local iterations in total, so $tE+e$ refers to the $e$-th local iteration in the communication round $t+1$. Moreover, $tE$ represents the time step before prototype aggregation at the server, and $tE+1/2$ represents the time step between prototype aggregation at the server and starting the first iteration on the local model.

\subsection{Assumptions}
\setcounter{assumption}{0}
\begin{assumption}\label{as1}
{\rm(Lipschitz Smooth).} Each local objective function is $L_1$-Lipschitz smooth, which also means the gradient of local objective function is $L_1$-Lipschitz continuous,
\begin{equation}
\small
\begin{aligned}
\|\nabla \mathcal{L}_{{t_1}}&-\nabla \mathcal{L}_{{t_2}}\|_2 \leq L_1\| \omega_{{i,t_1}}  -\omega_{{i,t_2}} \|_2, \\
&\forall t_1, t_2 >0, i \in \{1,2,\dots, m\},   
\end{aligned}
\end{equation}

\noindent which implies the following quadratic bound,
\begin{equation}
\begin{aligned}
\mathcal{L}_{{t_1}} - \mathcal{L}_{{t_2}} &\leq \langle \nabla \mathcal{L}_{{t_2}}, (\omega_{{i,t_1}}  -\omega_{{i,t_2}})\rangle + \frac{L_1}{2} \|  \omega_{{i,t_1}}  -\omega_{{i,t_2}} \|_2^2, \quad\forall t_1, t_2 >0, i \in \{1,2,\dots, m\}.
\end{aligned}
\end{equation}
\end{assumption}

\begin{assumption} \label{as2}
{\rm(Unbiased Gradient and Bounded Variance).} The stochastic gradient $g_{i,t} = \nabla \mathcal{L}(\omega_{i,t}, \xi_{t}) $ is an unbiased estimator of the local gradient for each client. Suppose its expectation 
\begin{equation}
\begin{aligned}
{\Bbb E}_{\xi_i \sim D_i}{[}g_{i,t}{]} = \nabla \mathcal{L}(\omega_{i,t}) = \nabla \mathcal{L}_{t}, \quad \forall i \in {1,2,\dots, m},
\end{aligned}
\end{equation}
\noindent and its variance is bounded by $\sigma^2$: 
\begin{equation}
\begin{aligned}
{\Bbb E}{[}{\| g_{i,t} - \nabla \mathcal{L}(\omega_{i,t}) \|}_2^2{]}\leq \sigma^2, \quad \forall i \in \{1,2,\dots, m\}, \sigma^2 \geq 0.
\end{aligned}
\end{equation}
\end{assumption}

\begin{assumption} \label{as3}
{\rm(Bounded Expectation of Euclidean norm of Stochastic Gradients).}The expectation of the stochastic gradient is bounded by $G$:
\begin{equation}
{\Bbb E}{[} \| g_{i,t} \|_2{]} \leq G, \quad \forall i \in \{1,2,\dots, m\}.
\end{equation}
\end{assumption}

\begin{assumption}\label{as4}
{\rm(Lipschitz Continuity).} 
Each local embedding function is $L_2$-Lipschitz continuous, that is,
\begin{equation}
\left\| f_i(\phi_{i,t_1})- f_i(\phi_{i,t_2})\right\| \leq L_2\| \phi_{i,t_1}  -\phi_{i,t_2} \|_2, \quad\forall t_1, t_2 >0, i \in \{1,2,\dots, m\}.
\end{equation}
\end{assumption}

\noindent Assumption \ref{as4} is a little strong, but we only use it in a very narrow domain with width of $E$ steps of SGD in Lemma \ref{lemma2}.

\subsection{Key Lemmas}
\setcounter{lemma}{0}
\begin{lemma} \label{lemma1}
Let Assumption \ref{as1} and \ref{as2} hold. From the beginning of communication round $t+1$ to the last local update step, the loss function of an arbitrary client can be bounded as:
\begin{equation}
\begin{split}
{\Bbb E} {[}\mathcal{L}_{(t+1)E}{]} \leq \mathcal{L}_{tE+1/2} -(\eta-\frac{L_1 \eta^2}{2}) \sum_{e=1/2}^{E-1} \|\nabla \mathcal{L}_{tE+e} \|_2^2+ \frac{L_1 E \eta^2}{2} \sigma^2.
\end{split}
\end{equation}
\end{lemma}

\begin{proof}
Due to the fact that this lemma is for an arbitrary client, so client notation $i$ is omitted. Let $\omega_{t+1} = \omega_{t} - \eta g_{t}$, then
\begin{equation}
\begin{split}
\mathcal{L}_{tE+1} & \stackrel{(a)}{\leq} \mathcal{L}_{tE+1/2} + \langle \nabla \mathcal{L}_{tE+1/2}, (\omega_{tE+1}  -\omega_{tE+1/2})\rangle + \frac{L_1}{2} \|  \omega_{tE+1}  -\omega_{tE+1/2} \|_2^2\\
&= \mathcal{L}_{tE+1/2}-\eta \langle \nabla \mathcal{L_1}_{tE+1/2}, g_{tE+1/2}\rangle + \frac{L_1}{2} \|  \eta g_{tE+1/2} \|_2^2,
\end{split}
\end{equation}
\noindent where (a) follows from the quadratic $L_1$-Lipschitz smooth bound in Assumption \ref{as1}. Taking expectation of both sides of the above equation on the random variable $\xi_{tE+1/2}$, we have
\begin{align}
{\Bbb E} {[}\mathcal{L}_{tE+1}{]} & \leq \mathcal{L}_{tE+1/2}-\eta {\Bbb E}{[}\langle \nabla \mathcal{L_1}_{tE+1/2},  g_{tE+1/2}\rangle{]} + \frac{L_1 \eta^2}{2} {\Bbb E}{[}\|g_{tE+1/2}\|_2^2{]}\\
&  \stackrel{(b)}{=} \mathcal{L}_{tE+1/2}-\eta \|\nabla \mathcal{L}_{tE+1/2} \|_2^2 + \frac{L_1 \eta^2}{2} {\Bbb E}{[}\|g_{i,tE+1/2}\|_2^2{]}\\
&  \stackrel{(c)}{\leq} \mathcal{L}_{tE+1/2}-\eta \|\nabla \mathcal{L}_{tE+1/2} \|_2^2 + \frac{L_1 \eta^2}{2} (\|\nabla \mathcal{L}_{tE+1/2} \|_2^2 + Var(g_{i,tE+1/2}))\\
&= \mathcal{L}_{tE+1/2} -(\eta-\frac{L_1 \eta^2}{2}) \|\nabla \mathcal{L}_{tE+1/2} \|_2^2 + \frac{L_1 \eta^2}{2} Var(g_{i,tE+1/2})\\
& \stackrel{(d)}{\leq} \mathcal{L}_{tE+1/2} -(\eta-\frac{L_1 \eta^2}{2}) \|\nabla \mathcal{L}_{tE+1/2} \|_2^2 + \frac{L_1 \eta^2}{2} \sigma^2, \label{eq:3}
\end{align}
\noindent where (b) follows from Assumption \ref{as2}, (c) follows from $Var(x)={\Bbb E}{[}x^2{]}-({\Bbb E {[}x{]}})^2$, (d) follows from Assumption \ref{as2}. Take expectation of $\omega$ on both sides. Then, by telescoping of $E$ steps, we have,
\begin{equation}
\begin{split}
{\Bbb E} {[}\mathcal{L}_{(t+1)E}{]} \leq \mathcal{L}_{tE+1/2} -(\eta-\frac{L_1 \eta^2}{2}) \sum_{e=1/2}^{E-1} \|\nabla \mathcal{L}_{tE+e} \|_2^2 + \frac{L_1 E \eta^2}{2} \sigma^2.
\end{split}
\end{equation}
\end{proof}
\vspace{-2cm}

\begin{lemma} \label{lemma2}
Let Assumption \ref{as3} and \ref{as4} hold. After the prototype aggregation at the server, the loss function of an arbitrary client can be bounded as:
\begin{equation}
{\Bbb E}{[}\mathcal{L}_{(t+1)E+1/2}{]} \leq \mathcal{L}_{(t+1)E}+\lambda L_2 \eta E G
\end{equation}
\end{lemma}

\begin{proof}
\begin{align}
    \mathcal{L}_{(t+1)E+1/2} &= \mathcal{L}_{(t+1)E}+\mathcal{L}_{(t+1)E+1/2}-\mathcal{L}_{(t+1)E} \\
    &\stackrel{(a)}{=} \mathcal{L}_{(t+1)E} + \lambda \|f_i(\phi_{i,(t+1)E}) - \bar{C}_{t+2} \|_2 - \lambda \|f_i(\phi_{i,(t+1)E}) - \bar{C}_{t+1} \|_2 \\
    & \stackrel{(b)}{\leq} \mathcal{L}_{(t+1)E} + \lambda \|\bar{C}_{t+2} - \bar{C}_{t+1} \|_2 \\ 
    &\stackrel{(c)}{=} \mathcal{L}_{(t+1)E}+ \lambda \|\sum_{i=1}^m q_i C_{i,(t+1)E} - \sum_{i=1}^m q_i C_{i,tE} \|_2 \\
    &= \mathcal{L}_{(t+1)E}+ \lambda \|\sum_{i=1}^m q_i (C_{i,(t+1)E} -  C_{i,tE}) \|_2 \\
    &\stackrel{(d)}{=}\mathcal{L}_{(t+1)E}+ \lambda \| \sum_{i=1}^m q_i  \frac{1}{|D_i|}\sum_{k=1}^{|D_i|}(f_i(\phi_{i,(t+1)E};x_{i,k}) - f_i(\phi_{i,tE};x_{i,k}) \|_2 \\
    &\stackrel{(e)}{\leq} \mathcal{L}_{(t+1)E}+ \lambda \sum_{i=1}^m \frac{q_i}{|D_i|}\sum_{k=1}^{|D_i|} \| f_i(\phi_{i,(t+1)E};x_{i,k}) -  f_i(\phi_{i,tE};x_{i,k}) \|_2\\
    &\stackrel{(f)}{\leq} \mathcal{L}_{(t+1)E}+ \lambda L_2 \sum_{i=1}^m {q_i} \| \phi_{i,(t+1)E} -  \phi_{i,tE} \|_2 \\
    &\stackrel{(g)}{\leq} \mathcal{L}_{(t+1)E}+ \lambda L_2 \sum_{i=1}^m {q_i} \| \omega_{i,(t+1)E} -  \omega_{i,tE} \|_2 \\
    & = \mathcal{L}_{(t+1)E}+\lambda L_2 \eta \sum_{i=1}^m q_i \|\sum_{e=1/2}^{E-1} g_{i,tE+e} \|_2\\
    & \stackrel{(h)}{\leq} \mathcal{L}_{(t+1)E}+\lambda L_2 \eta \sum_{i=1}^m q_i \sum_{e=1/2}^{E-1} \| g_{i,tE+e} \|_2\\
\end{align}

\noindent Take expectations of random variable $\xi$ on both sides, then
\begin{align}
    {\Bbb E}{[}\mathcal{L}_{(t+1)E+1/2} {]}& {\leq} \mathcal{L}_{(t+1)E}+\lambda L_2 \eta \sum_{i=1}^m q_i \sum_{e=1/2}^{E-1} {\Bbb E}{[} \| g_{i,tE+e} \|_2{]}\\
    & \stackrel{(i)}{\leq} \mathcal{L}_{(t+1)E}+\lambda L_2 \eta E G,
\end{align}

\noindent where (a) follows from the definition of local loss function in Eq. \ref{eq:local_re}, (b) follows from $\|a-b\|_2 - \|a-c\|_2 \leq \|b-c\|_2$, (c) follows from the definition of global prototype in Eq. \ref{eq:global_proto}, (d) follows from the definition of local prototype in Eq. \ref{eq:local_proto}, (e) and (h) follow from $\| \sum a_i \|_2 \leq \sum {\|a_i\|_2}$, (f) follows from $L_2$-Lipschitz continuity in Assumption \ref{as4}, (g) follows from the fact that $\phi_i$ is a subset of $\omega_i$, (i) follows from Assumption \ref{as3}.
\end{proof}

\subsection{Theorems}
\setcounter{theorem}{0}
\begin{theorem} \label{th1}
{\rm (One-round deviation).} Let Assumption 1 to 4 hold. For an arbitrary client, after every communication round, we have,
\begin{equation} \label{eq: th1}
\begin{split}
{\Bbb E} [\mathcal{L}_{(t+1)E+1/2}] \leq \mathcal{L}_{tE+1/2} -(\eta-\frac{L_1 \eta^2}{2}) \sum_{e=1/2}^{E-1} \|\nabla \mathcal{L}_{tE+e} \|_2^2 + \frac{L_1 E \eta^2}{2} \sigma^2 + \lambda L_2 \eta E G.
\end{split}
\end{equation}
\end{theorem}

\setcounter{corollary}{0}
\begin{corollary} \label{co1}
{\rm (Non-convex \algfont{FedProto} convergence).} 
\noindent The loss function $\mathcal{L}$ of arbitrary client monotonously decreases in every communication round when
\begin{equation}
\eta_{e'}<\frac{2( \sum_{e=1/2}^{e'} \|\nabla \mathcal{L}_{tE+e} \|_2^2- \lambda L_2 EG)}{L_1(\sum_{e=1/2}^{e'} \|\nabla \mathcal{L}_{tE+e} \|_2^2+E\sigma^2)},\ e'=1/2,1,\dots, E-1
\end{equation}
and
\begin{equation}
\lambda_t < \frac{ \|\nabla \mathcal{L}_{tE+1/2} \|_2^2}{L_2 EG }.
\end{equation}
Thus, the loss function converges.
\end{corollary}

\begin{theorem} \label{th2}
{\rm (Non-convex convergence rate of \algfont{FedProto})}. Let Assumption 1 to 4 hold and $\Delta =\mathcal{L}_0-\mathcal{L}^* $. For an arbitrary client, given any $\epsilon>0$, after 
\begin{equation}
T = \frac{2\Delta}{E \epsilon (2\eta - L_1 \eta^2) - E \eta (L_1 \eta \sigma^2 + 2\lambda L_2 G)}
\end{equation} 
communication rounds of \algfont{FedProto}, we have 
\begin{equation}
\frac{1}{T E}\sum_{t=0}^{T-1} \sum_{e=1/2}^{E-1}{\Bbb E}[\|\nabla \mathcal{L}_{tE+e} \|_2^2] < \epsilon,
\end{equation}
if
\begin{equation}
\eta <  \frac{2(\epsilon-\lambda L_2 G)}{L_1(\epsilon+\sigma^2)},
\end{equation} 
and
\begin{equation}
\lambda <  \frac{\epsilon}{L_2 G}.
\end{equation} 
\end{theorem}

\subsection{Completing the Proof of Theorem \ref{th1} and Corollary \ref{co1}}
\begin{proof}
Taking expectation of $\omega$ on both sides in Lemma \ref{lemma1} and \ref{lemma2}, then sum them, we can easily get
\begin{equation} \label{eq35}
\begin{split}
{\Bbb E} [\mathcal{L}_{(t+1)E+1/2}] \leq \mathcal{L}_{tE+1/2} -(\eta-\frac{L_1 \eta^2}{2}) \sum_{e=1/2}^{E-1} \|\nabla \mathcal{L}_{tE+e} \|_2^2 + \frac{L_1 E \eta^2}{2} \sigma^2 + \lambda L_2 \eta E G
\end{split}
\end{equation}

\noindent Then, to make sure $-(\eta-\frac{L_1 \eta^2}{2}) \sum_{e=1/2}^{E-1} \|\nabla \mathcal{L}_{tE+e} \|_2^2 + \frac{L_1 E \eta^2}{2} \sigma^2 + \lambda L_2 \eta E G \leq 0$, we get
\begin{equation}
\eta<\frac{2( \sum_{e=1/2}^{E-1} \|\nabla \mathcal{L}_{tE+e} \|_2^2- \lambda L_2 E G)}{L_1(\sum_{e=1/2}^{E-1} \|\nabla \mathcal{L}_{tE+e} \|_2^2+E \sigma^2)},
\end{equation}
and
\begin{equation}
\lambda < \frac{\sum_{e=1/2}^{E-1} \|\nabla \mathcal{L}_{tE+e} \|_2^2}{L_2 E G}.
\end{equation}
\noindent In practice, we use
\begin{equation}
\eta_{e'}<\frac{2( \sum_{e=1/2}^{e'} \|\nabla \mathcal{L}_{tE+e} \|_2^2- \lambda L_2 EG)}{L_1(\sum_{e=1/2}^{e'} \|\nabla \mathcal{L}_{tE+e} \|_2^2+E\sigma^2)},\ e'=1/2,1,\dots, E-1
\end{equation}
and
\begin{equation}
\lambda_t < \frac{ \|\nabla \mathcal{L}_{tE+1/2} \|_2^2}{L_2 EG }.
\end{equation}
So, the convergence of $\mathcal{L}$ holds.
\end{proof}

\subsection{Completing the Proof of Theorem \ref{th2}}
\begin{proof}
Take expectation of $\omega$ on both sides in Eq. \ref{eq: th1}, then telescope considering the communication round from $t=0$ to $t=T-1$ with the timestep from $e=1/2$ to $t=E$ in each communication round, we have

\begin{align}
\frac{1}{T E}\sum_{t=0}^{T-1} \sum_{e=1/2}^{E-1}{\Bbb E}[\|\nabla \mathcal{L}_{tE+e} \|_2^2] &\leq \frac{\frac{1}{T E}\sum_{t=0}^{T-1}(\mathcal{L}_{tE+1/2} - {\Bbb E} [\mathcal{L}_{(t+1)E+1/2}]) + \frac{L_1 \eta^2}{2} \sigma^2 + \lambda L_2 \eta  G}{\eta-\frac{L_1 \eta^2}{2}}.
\end{align}

\noindent Given any $\epsilon >0$, let
\begin{equation}
\frac{\frac{1}{T E}\sum_{t=0}^{T-1}(\mathcal{L}_{tE+1/2} - {\Bbb E} [\mathcal{L}_{(t+1)E+1/2}]) + \frac{L_1 \eta^2}{2} \sigma^2 + \lambda L_2 \eta  G}{\eta-\frac{L_1 \eta^2}{2}} < \epsilon,
\end{equation}

\noindent that is
\begin{equation}
\frac{\frac{2}{T E}\sum_{t=0}^{T-1}(\mathcal{L}_{tE+1/2} - {\Bbb E} [\mathcal{L}_{(t+1)E+1/2}]) + {L_1 \eta^2} \sigma^2 + 2\lambda L_2 \eta  G}{2\eta-{L_1 \eta^2}} < \epsilon.
\end{equation}

\noindent Let $\Delta =\mathcal{L}_0-\mathcal{L}^* $. Since $\sum_{t=0}^{T-1}(\mathcal{L}_{tE+1/2} - {\Bbb E} [\mathcal{L}_{(t+1)E+1/2}]) \leq \Delta$, the above equation holds when
\begin{align}
\frac{\frac{2\Delta}{T E} + {L_1 \eta^2} \sigma^2 + 2\lambda L_2 \eta  G}{2\eta-{L_1 \eta^2}} < \epsilon,
\end{align}

\noindent that is
\begin{align}
T > \frac{2\Delta}{E \epsilon (2\eta - L_1 \eta^2) - E \eta (L_1 \eta \sigma^2 + 2\lambda L_2 G)}.
\end{align}

\noindent So, we have
\begin{equation}
\frac{1}{T E}\sum_{t=0}^{T-1} \sum_{e=1/2}^{E-1}{\Bbb E}[\|\nabla \mathcal{L}_{tE+e} \|_2^2] < \epsilon,
\end{equation}

\noindent when
\begin{equation}
\eta <  \frac{2(\epsilon-\lambda L_2 G)}{L_1(\epsilon+\sigma^2)},
\end{equation} 
and
\begin{equation}
\lambda <  \frac{\epsilon}{L_2 G}.
\end{equation}
\end{proof}

\end{document}